\documentclass[11pt,letterpaper]{article}
\usepackage{amssymb,amsmath,amsthm,enumerate,nicefrac}
\usepackage{graphicx, color}
    \usepackage[driverfallback=hypertex,pagebackref=false,colorlinks]{hyperref}
    \hypersetup{linkcolor=[rgb]{.7,0,0}}
    \hypersetup{citecolor=[rgb]{0,.7,0}}
    \hypersetup{urlcolor=[rgb]{.7,0,.7}}
\usepackage{geometry}
\usepackage{epstopdf}
\AtBeginDocument{%
  \addtolength\abovedisplayskip{-0.25\baselineskip}%
  \addtolength\belowdisplayskip{-0.15\baselineskip}%
}

\usepackage{float}
\floatstyle{ruled}
\newfloat{algorithm}{tbp}{loa}
\providecommand{\algorithmname}{Algorithm}
\floatname{algorithm}{\protect\algorithmname}

\usepackage[titletoc,title]{appendix}
\usepackage{url}

\theoremstyle{plain}
\newtheorem{theorem}{Theorem}
\newtheorem{lemma}{Lemma}[section]
\newtheorem{claim}[lemma]{Claim}
\newtheorem{conjecture}[lemma]{Conjecture}
\newtheorem{definition}[lemma]{Definition}
\newtheorem{corollary}[theorem]{Corollary}

\def\XX{ X }
\def\AA{ A }

\def\N{{\mathbb {N}}}
\def\Reals{{\mathbb {R}}}

\newcommand{\Ex}{\mathop{\bf E\/}}

\def\U{{\mathcal {U}}}

\def\T{{\mathcal {T}}}

\def\P{{\mathbb {P}}}
\newcommand{\norm}[1]{\left\lVert#1\right\rVert}
\newcommand{\inner}[1]{\langle#1\rangle}

\def\Bad{\mathrm{Bad}}
\def\Sig{\mathrm{Sig}}

\usepackage{color}

\newcommand{\kk}{k}
\newcommand{\rr}{r}

\newcommand{\F}{\mathbb{F}}

\newcommand{\vareps}{\varepsilon}

\title{Memory-Sample Lower Bounds for Learning Parity with Noise}
\author{Sumegha Garg\thanks{sumegha.garg@gmail.com. Department of Computer Science, Harvard University. Research supported by Michael O. Rabin Postdoctoral Fellowship.}
\and
Pravesh K. Kothari\thanks{
kotpravesh@gmail.com. Department of Computer Science, Carnegie Mellon University. Research supported by NSF CAREER Award \#2047933.}
\and     Pengda Liu%
\thanks{pengda@stanford.edu. Department of Computer Science, Stanford University.}
    \and
    Ran Raz%
\thanks{ran.raz.mail@gmail.com. Department of Computer Science, Princeton University. Research supported by the Simons Collaboration on Algorithms and Geometry, by a Simons Investigator Award and by the National Science Foundation grants No. CCF-1714779, CCF-2007462.}
    }

\date{}
\begin{document}
\maketitle

\begin{abstract}

In this work, we show, for the well-studied problem of learning parity under noise, where a learner tries to learn $x=(x_1,\ldots,x_n) \in \{0,1\}^n$ from a stream of random linear equations over $\F_2$ that are correct with probability $\frac{1}{2}+\vareps$ and flipped with probability $\frac{1}{2}-\vareps$ ($0<\vareps< \frac{1}{2}$), that any learning algorithm requires either a memory of size $\Omega(n^2/\vareps)$ or an exponential number of samples. 

In fact, we study memory-sample lower bounds for a large class of learning problems, as characterized by~\cite{GRT18}, when the samples are noisy. 
A matrix $M: \AA \times \XX \rightarrow \{-1,1\}$ corresponds to the following learning problem with error parameter $\vareps$: an unknown element $x \in \XX$ is chosen uniformly at random. A learner tries to learn $x$ from a stream of samples,
$(a_1, b_1), (a_2, b_2) \ldots$, where for every $i$, $a_i \in \AA$ is
chosen uniformly at random and $b_i = M(a_i,x)$ with probability $1/2+\vareps$ and $b_i = -M(a_i,x)$ with probability $1/2-\vareps$ ($0<\vareps< \frac{1}{2}$).
Assume that $k,\ell, r$ are such that any submatrix of $M$ of at least $2^{-k} \cdot |A|$ rows and at least
$2^{-\ell} \cdot |X|$ columns, has a bias of at most $2^{-r}$.
We show that
any learning algorithm for the learning problem corresponding to $M$, with error parameter $\varepsilon$, requires either a memory
of size at least
$\Omega\left(\frac{k \cdot  \ell}{\vareps} \right)$, or at least $2^{\Omega(r)}$ samples.
The result holds even if the learner has an exponentially small success probability (of $2^{-\Omega(r)}$). In particular, this shows that for a large class of learning problems, same as those in~\cite{GRT18}, any learning algorithm requires either a memory
of size at least
$\Omega\left(\frac{(\log |X|) \cdot  (\log |A|)}{\vareps}\right)$  or an exponential number of noisy samples.

Our proof is based on adapting the arguments in~\cite{Raz17,GRT18} to the noisy case.

%

\end{abstract}
\thispagestyle{empty}
\clearpage
\setcounter{page}{1}

\section{Introduction}
In this work, we study the number of samples needed for learning under noise and memory constraints. The study of the resources needed for learning, under
memory constraints
was initiated by Shamir~\cite{Shamir} and
Steinhardt, Valiant and Wager~\cite{SVW}, and has been studied in the streaming setting.
In addition to being a natural question in learning theory and complexity theory, lower bounds in this model also have direct applications to bounded storage cryptography~\cite{Raz16, VV, KRT, Stefano18,Zhandry19,  Stefano19, Stefano20,GZ21}. \cite{SVW} conjectured that any algorithm for learning parities of size $n$ (that is, learning $x\in\{0,1\}^n$ from a stream of random linear equations in $\F_2$)
requires either a memory
of size $\Omega(n^2)$ or an exponential number of samples.
This conjecture was proven in~\cite{Raz16} and in follow up works, this was generalized to learning sparse parities in \cite{KRT} and more general learning problems in~\cite{Raz17,MM17,MT,GRT18,BGY18,DS,MM2,SSV19,GRT19,dks19,girish20}. 

In this work, we extend this line of work to \emph{noisy} Boolean function learning problems. In particular, we consider the well-studied problem of learning parity under noise (LPN). In this problem, a learner wants to learn $x\in\{0,1\}^n$ from independent and uniformly random linear equations in $\F_2$ where the right hand sides are obtained by independently flipping the evaluation of an unknown parity function with probability $\frac{1}{2}-\vareps$. Learning Parity with Noise (LPN) is a central problem in Learning and Coding Theory (often referred to as decoding random linear codes) and has been extensively studied. Even without memory constraints, coming up with algorithms for the problem has proven to be challenging and the current state-of-the-art for solving the problem is still the celebrated work of Blum, Kalai and Wasserman~\cite{DBLP:journals/jacm/BlumKW03} that runs in time $2^{O(n/\log_2(n))}$. Over time, the hardness of LPN (and its generalization to non-binary finite fields) has been used as a starting point in several hardness results~\cite{DBLP:journals/siamcomp/KalaiKMS08,DBLP:journals/siamcomp/FeldmanGKP09} and constructing cryptographic primitives~\cite{DBLP:conf/focs/Alekhnovich03}. On the other hand, lower-bounds for the problem are known only in restricted models such as Statistical Query Learning\footnote{The SQ model does not seem to distinguish between noisy and noiseless variants of parity learning and yields the same lower bound in both cases.}~\cite{sq}.

Learning under noise is at least as hard as learning without noise and thus, memory-sample lower bounds for parity learning~\cite{Raz16} holds for learning parity under noise too. It is natural to ask -- can we get better space lower bounds for learning parities under noise? In this work, we are able to strengthen the memory lower bound to $\Omega(n^2/\vareps)$ for parity learning with noise. 

Our results actually extend to a broad class of learning problems under noise. As in~\cite{Raz17} and follow up works, we represent a learning problem using a matrix.
Let $\XX$, $\AA$ be two finite sets (where $\XX$ represents the concept-class that we are trying to learn and $\AA$ represents the set of possible samples).
Let $M: \AA \times \XX \rightarrow \{-1,1\}$ be a matrix.
The matrix $M$ represents the following learning problem with error parameter $\vareps$ ($0<\vareps<\frac{1}{2}$):
An unknown element $x \in \XX$ was chosen uniformly at random. A learner tries to learn $x$ from a stream of samples,
$(a_1, b_1), (a_2, b_2) \ldots$, where for every $i$, $a_i \in \AA$ is
chosen uniformly at random and $b_i = M(a_i,x)$ with probability $\frac{1}{2}+\vareps$.

\subsection*{Our Results}

We use extractor-based characterization of the matrix $M$ to prove our lower bounds, as done in~\cite{GRT18}. Our main result can be stated as follows (Corollary~\ref{cor:main1}):
Assume that $k,\ell, r$ are such that any submatrix of $M$ of at least $2^{-k} \cdot |A|$ rows and at least
$2^{-\ell} \cdot |X|$ columns, has a bias of at most $2^{-r}$.
Then,
any learning algorithm for the learning problem corresponding to $M$ with error parameter $\vareps$ requires either a memory
of size at least
$\Omega\left(k \cdot  \ell/\vareps \right)$, or at least $2^{\Omega(r)}$ samples.
Thus, we get an extra factor of $\frac{1}{\vareps}$ in the space lower bound for all the bounds on learning problems that~\cite{GRT18} imply, some of which are as follows (see~\cite{GRT18} for details on why the corresponding matrices satisfy the extractor-based property):
\begin{enumerate}
\item {\bf Parities with noise:}
A learner tries to learn $x=(x_1,\ldots,x_n) \in \{0,1\}^n$, from (a stream of) random linear equations over $\F_2$ which are correct with probability $\frac{1}{2}+\vareps$ and flipped with probability $\frac{1}{2}-\vareps$. Any learning
algorithm
requires either a memory
of size $\Omega(n^2/\vareps)$ or an exponential number of samples.
\item {\bf Sparse parities with noise:}
A learner tries to learn $x=(x_1,\ldots,x_n) \in \{0,1\}^n$ of sparsity $\ell$, from (a stream of) random linear equations over $\F_2$ which are correct with probability $\frac{1}{2}+\vareps$ and flipped with probability $\frac{1}{2}-\vareps$.
Any learning algorithm  requires:
\begin{enumerate}
	\item Assuming $\ell \le n/2$: either a memory
of size $\Omega(n \cdot \ell/\vareps)$ or $2^{\Omega(\ell)}$ samples.
	\item Assuming $\ell \le n^{0.9}$:  either a memory
of size  $\Omega(n \cdot \ell^{0.99}/\vareps)$ or $\ell^{\Omega(\ell)}$ samples.
\end{enumerate}
\item {\bf Learning from noisy sparse linear equations:}
A learner tries to learn $x=(x_1,\ldots,x_n) \in \{0,1\}^n$, from (a stream of) random sparse linear equations, of sparsity $\ell$, over $\F_2$, which are correct with probability $\frac{1}{2}+\vareps$ and flipped with probability $\frac{1}{2}-\vareps$.
Any learning algorithm requires:
\begin{enumerate}
	\item Assuming $\ell \le n/2$: either a memory
of size $\Omega(n \cdot \ell/\vareps)$ or $2^{\Omega(\ell)}$ samples.
	\item Assuming $\ell \le n^{0.9}$:  either a memory
of size  $\Omega(n \cdot \ell^{0.99}/\vareps)$ or $\ell^{\Omega(\ell)}$ samples.
\end{enumerate}
\item {\bf Learning from noisy low-degree equations:}
A learner tries to learn $x=(x_1,\ldots,x_n) \in \{0,1\}^n$, from (a stream of) random multilinear polynomial  equations of degree at most $d$, over $\F_2$, which are correct with probability $\frac{1}{2}+\vareps$ and flipped with probability $\frac{1}{2}-\vareps$.
We prove that if $d\le 0.99 \cdot n$, any learning algorithm
requires either a memory
of size
$\Omega\left( \binom{n}{\le d} \frac{n}{d\cdot \vareps} \right)$
or
$2^{\Omega(n/d)}$
samples (where $\binom{n}{\le d} = \binom{n}{0} + \binom{n}{1} + \ldots + \binom{n}{d}$).
\item {\bf Low-degree polynomials with noise:}
A learner tries to learn an $n'$-variate multilinear polynomial $p$ of degree at most $d$ over $\F_2$,
from (a stream of) random evaluations of $p$ over $\F_2^{n'}$, which are correct with probability $\frac{1}{2}+\vareps$ and flipped with probability $\frac{1}{2}-\vareps$.
We prove that if $d\le 0.99 \cdot n'$,
    any learning algorithm
requires  either a memory
of size
$\Omega\left( \binom{n'}{\le d} \cdot  \frac{n'}{d\cdot \vareps} \right)$
or
$2^{\Omega(n'/d)}$
samples.

\end{enumerate}

\subsection*{Techniques}

Our proof follows the proof of~\cite{Raz17,GRT18} very closely and builds on that proof. We extend the extractor-based result of ~\cite{GRT18} to the noisy case and a straightforward adaptation to its proof gives the stronger lower bound for the noisy case (which reflects on the strength of the current techniques). The main contribution of this paper is not a technical one but establishing stronger space lower bounds for a well-studied problem of learning parity with noise, using the current techniques.

\subsection*{Discussion and Open Problem}
Let's look at a space upper bound for the problem of learning parity with noise, that is, a learner tries to learn $x\in \{0,1\}^n$ from a stream of samples of the form
$(a, b)$, where $a \in\{0,1\}^n$ is chosen uniformly at random and $b =a\cdot x$ with probability $\frac{1}{2}+\vareps$ and $b = 1-a\cdot x$ with probability $\frac{1}{2}-\vareps$ (here, $a\cdot x$ represents the inner product of $a$ and $x$ in $\F_2$, that is, $a\cdot x=\sum_i a_i x_i \mod 2$).

\paragraph{Upper Bound:} Consider the following algorithm $A$: Store the first $m=O(n/\vareps^2)$ samples. Check for every $x'\in\{0,1\}^n$, if for at least $\left(\frac{1}{2}+\frac{\vareps}{2}\right)$ fraction of the samples $(a_1,b_1),\ldots, (a_m,b_m)$, $a_i\cdot x'$ agrees with $b_i$. Output the first $x'$ that satisfies the check. In expectation, $a_i\cdot x$ would agree with $b_i$ for $\left(\frac{1}{2}+\vareps\right)$ fraction of the samples, and otherwise for $x'\neq x$, in expectation, $a_i\cdot x'$ would agree with $b_i$ for half the samples.
Therefore, for large enough $m$, using Chernoff bound and a union bound, with high probability ($1-o(1)$) over the $m$ samples, $x'$ satisfies the check if and only if $x'=x$, and $A$ outputs the correct answer under such an event. $A$ uses $O(n/\vareps^2)$ samples and $O(n^2/\vareps^2)$ bits of space.

In this paper, we prove that any algorithm that learns parity with noise from a stream of samples (as defined above) requires $\Omega(n^2/\vareps)$ bits of space or exponential number of samples. Improving the lower bound to match the upper bound (or vice versa) is a fascinating open problem and we conjecture that the upper bound is tight.  As each sample gives at most $O(\vareps^2)$ bits of information about $x$, we can at least show that a learning algorithm requires $O(n/\vareps^2)$ samples to learn $x$ (which corresponds to using $O(n^2/\vareps^2)$ bits of space if each sample is stored). 

\begin{conjecture}
Any learner that tries to learn $x\in \{0,1\}^n$ from a stream of samples of the form
$(a, b)$, where $a \in\{0,1\}^n$ is chosen uniformly at random and $b =a\cdot x$ with probability $\frac{1}{2}+\vareps$ and $b = 1-a\cdot x$ with probability $\frac{1}{2}-\vareps$, requires either $\Omega(n^2/\vareps^2)$ bits of memory or $2^{\Omega(n)}$ samples.
\end{conjecture} 
The proof of the conjecture, if true, would lead to new technical insights (beyond extractor-based techniques) into proving time-space (or memory-sample) lower bounds for learning problems.

\subsection*{Outline of the Paper} In Section \ref{sec:prelim}, we establish certain notations and definitions, which are borrowed from~\cite{Raz17,GRT18}. We give a proof overview in Section \ref{sec:overview} and prove the main theorem in Section \ref{sec:main-result}. 
\section{Preliminaries}\label{sec:prelim}

Denote by $\U_X: \XX \rightarrow \Reals^+$ the uniform distribution over $\XX$.
Denote by $\log$ the logarithm to  base $2$.
For a random variable $Z$ and an event $E$,
we denote by $\P_Z$ the distribution of the random variables $Z$, and
we denote by $\P_{Z|E}$ the distribution of the random variable $Z$ conditioned on the event $E$.

\subsubsection*{Viewing a Learning Problem, with error $\frac{1}{2}-\vareps$, as a Matrix}

Let $\XX$, $\AA$ be two finite sets of size larger than 1.
Let $n = \log_2|\XX|$ and $n' = \log_2|\AA|$.

Let $M: \AA \times \XX \rightarrow \{-1,1\}$ be a matrix. The matrix $M$ corresponds to the following learning problem with error parameter $\vareps$ ($0<\vareps< \frac{1}{2}$).
There is an unknown element $x \in \XX$ that was chosen uniformly at random. A learner tries to learn $x$ from samples
$(a, b)$, where $a \in \AA$ is chosen uniformly at random, and $b = M(a,x)$ with probability $\frac{1}{2}+\vareps$ and $b = -M(a,x)$ with probability $\frac{1}{2}-\vareps$.
That is, the learning algorithm is given a stream of samples,
$(a_1, b_1), (a_2, b_2) \ldots$, where each~$a_t$ is uniformly distributed, and $b_t = M(a_t,x)$ with probability $\frac{1}{2}+\vareps$ and $b =-M(a_t,x)$ with probability $\frac{1}{2}-\vareps$.

\subsubsection*{Norms and Inner Products}
Let $p \geq 1$.
For a function
$f: \XX \rightarrow \Reals$,
denote by $\norm{f}_{p}$ the $\ell_p$ norm of $f$, with respect to the  uniform distribution over $\XX$, that is:
$$\norm{f}_{p} =
\left( \Ex_{x \in_R \XX} \left[ |f(x)|^{p} \right] \right)^{1/p}.$$
%
%

For two functions
$f,g: \XX \rightarrow \Reals$, define their inner product with respect to the uniform distribution over $X$ as
$$\langle f,g \rangle =
\Ex_{x \in_R \XX} [ f(x) \cdot g(x) ].$$

For a matrix $M: \AA \times \XX \to \Reals$ and a row $a \in \AA$, we denote by $M_a: \XX \to \Reals$ the function corresponding to the $a$-th row of $M$. Note that for a function $f: \XX \to \Reals$, we have $\inner{M_a, f} = \frac{(M \cdot f)_a}{|X|}$. Here, $M \cdot f$ represents the matrix multiplication of $M$ with $f$.

\subsubsection*{$L_2$-Extractors and $L_\infty$-Extractors} 

\begin{definition} {\bf 	$L_2$-Extractor:} \label{definition:l2-extractor}
Let $\XX,\AA$ be two finite sets.
A matrix $M: \AA \times \XX \to \{-1,1\}$ is a $(k,\ell)$-$L_2$-Extractor with error $2^{-r}$, if for every non-negative $f : \XX \to \Reals$ with $\frac{\norm{f}_2}{\norm{f}_1} \le 2^{\ell}$ there are at most $2^{-k} \cdot |A|$ rows $a$ in $A$ with
$$
\frac{|\inner{M_a,f}|}{\norm{f}_1}
\ge 2^{-r}\;.
$$
\end{definition}

Let $\Omega$ be a finite set. We denote a distribution over $\Omega$ as a function $f:\Omega \to \Reals^{+}$ such that $\sum_{x\in \Omega}{f(x)} = 1$.
We say that a distribution $f:\Omega \to \Reals^{+}$ has min-entropy $k$ if for all $x\in \Omega$, we have $f(x) \le 2^{-k}$.

\begin{definition}{\bf $L_\infty-$Extractor:} \label{definition:min-extractor}
Let $\XX,\AA$ be two finite sets.
A matrix $M:\AA\times \XX\rightarrow \{-1,1\}$ is a $\left(k ,\ell \sim r \right)$-$L_\infty$-Extractor
if for every distribution $p_x: \XX \to \Reals^{+}$ with min-entropy at least $(\log(|\XX|)-\ell)$
and every distribution $p_a: \AA \to \Reals^{+}$ with min-entropy at least $(\log(|\AA|)-k)$,
$$\bigg|\sum_{a'\in \AA} \sum_{x' \in \XX} p_a(a') \cdot p_x(x') \cdot M(a',x')\bigg| \le 2^{-r}.$$
\end{definition}

\subsubsection*{Branching Program for a Learning Problem} \label{section:def}

In the following definition, we model the learner for the learning problem that corresponds to the matrix $M$, by a {\em branching program}, as done by previous papers starting with~\cite{Raz16}.

\begin{definition} {\bf Branching Program for a Learning Problem:}
A branching program of length $m$ and width $d$, for learning, is a directed (multi) graph with vertices arranged in $m+1$ layers containing at most $d$ vertices each. In the first layer, that we think of as layer~0, there is only one vertex, called the start vertex.
A vertex of outdegree~0 is called a  leaf.
All vertices in the last layer are leaves
(but there may be additional leaves).
Every non-leaf vertex in the program has $2|\AA|$ outgoing edges, labeled by elements
$(a,b) \in \AA \times \{-1,1\}$, with exactly one edge labeled by each such $(a,b)$, and all these edges going
into vertices in the next layer.
Each leaf $v$ in the program is labeled by an element $\tilde{x}(v) \in \XX$, that
we think of as the output of the program on that leaf.

{\bf Computation-Path:} The samples
$(a_1, b_1), \ldots, (a_m, b_m) \in \AA \times \{-1,1\}$
that are given as input,
define a
computation-path in the branching
program, by starting from the start vertex
and following at
step~$t$ the edge labeled by~$(a_t, b_t)$, until reaching a leaf.
The program outputs the label $\tilde{x}(v)$ of the leaf $v$ reached by the computation-path.

{\bf Success Probability:}
The success probability of the program is the probability that $\tilde{x}=x$,
where $\tilde{x}$ is the element that the program outputs, and the probability is over $x,a_1,\ldots,a_m,b_1,\ldots,b_m$ (where $x$ is uniformly distributed over $\XX$ and $a_1,\ldots,a_m$ are uniformly distributed over $\AA$, and for every $t$, $b_t = M(a_t,x)$ with probability $\frac{1}{2}+\vareps$ and $-M(a_t,x)$ with probability $\frac{1}{2}-\vareps$).

\end{definition}
A learning algorithm, using $m$ samples and a memory of $s$ bits, can be modeled as a branching program\footnote{The lower bound holds for randomized learning algorithms because a branching program is a non-uniform model of computation, and we can fix a \emph{good} randomization for the computation without affecting the width.} of length $m$ and width $2^{O(s)}$. Thus, we will focus on proving width-length tradeoffs for any branching program that learns an extractor-based learning problem with noise, and such tradeoffs would translate into memory-sample tradeoffs for the learning algorithms.
\section{Overview of the Proof} \label{sec:overview}

The proof adapts the extractor-based time-space lower bound of~\cite{GRT18} to the \emph{noisy} case, which in turn built on~\cite{Raz17} that gave a general  technique for proving memory-samples lower bounds. We recall the arguments in~\cite{Raz17,GRT18} for convenience. 

Assume that $M$ is a $(k',\ell')$-$L_2$-extractor with error $2^{-r'}$,
and let $r = \min\{k', \ell', r'\}$.
Let $B$ be a branching program
for the noisy learning problem that corresponds to the matrix $M$. We want to prove that $B$ has at least $2^{\Omega(r)}$ length or requires at least $2^{\Omega(\frac{ k' \ell'}{\vareps})}$ width (that is, any learning algorithm solving the learning problem corresponding to the matrix $M$ with error parameter $\vareps$, requires either $\Omega(\frac{ k' \ell'}{\vareps})$ memory or exponential number of samples).
Assume for a contradiction that $B$ is
of length $m=2^{c r}$ and width $d=2^{c \frac{ k' \ell'}{\vareps}}$,
where $c>0$ is a  small constant.

We define the {\it truncated-path}, $\T$, to be the same as the computation-path of $B$, except that it sometimes stops before reaching a leaf.
Roughly speaking, $\T$ stops before reaching a leaf if certain ``bad'' events occur.
Nevertheless, we show that the probability that $\T$ stops before reaching a leaf is negligible, so we can think of $\T$ as almost identical to the computation-path.

For a vertex $v$ of $B$, we denote by $E_v$
the event that $\T$ reaches the vertex $v$.
We denote  by $\Pr(v) = \Pr(E_v)$ the probability for $E_v$
(where the probability is over $x,a_1,\ldots,a_m, b_1,\ldots,b_m$), and we denote
by $\P_{x|v} = \P_{x|E_v}$ the distribution of the random variable $x$ conditioned on the event~$E_v$.
Similarly,
for an edge~$e$ of the branching program $B$, let $E_e$ be
the event that $\T$ traverses the edge~$e$.
Denote, $\Pr(e) = \Pr(E_e)$, and
$\P_{x|e} = \P_{x|E_e}$.

A vertex $v$ of $B$ is called {\em significant} if
$$
\norm{\P_{x|v}}_{2} > 2^{\ell'} \cdot 2^{-n}.
$$
Roughly speaking, this means that conditioning on the event that $\T$ reaches the
vertex~$v$, a non-negligible amount of information is known about $x$.
In order to guess $x$ with a non-negligible success probability, $\T$ must reach a significant vertex. Lemma~\ref{lemma-main1} shows that the probability that $\T$ reaches any significant vertex is negligible, and thus the main result follows.

To prove Lemma~\ref{lemma-main1}, we show that for every fixed significant vertex $s$, the probability that $\T$ reaches $s$ is at most
$2^{-\Omega(k' \ell'/\vareps)}$ (which is smaller than one over the number of vertices
in~$B$). Hence, we can use a union bound to prove the lemma.

The proof  that the probability that $\T$ reaches $s$ is extremely small is the main
part of the proof.
To that end, we use the following functions to measure the progress made by the branching program towards reaching $s$.

Let $L_i$ be the set of vertices $v$ in layer-$i$ of $B$,
such that $\Pr (v) >0$. Let $\Gamma_i$ be the set of edges $e$ from layer-$(i-1)$ of $B$ to layer-$i$ of $B$,
such that $\Pr (e) >0$.
Let
$$
{\cal Z}_i =
\sum_{v \in L_i} \Pr(v) \cdot \langle \P_{x|v},\P_{x|s} \rangle^{k'/2\vareps},
$$
$$
{\cal Z}'_i =
\sum_{e \in \Gamma_i} \Pr(e) \cdot \langle \P_{x|e},\P_{x|s} \rangle^{k'/2\vareps}.
$$
We think of ${\cal Z}_i, {\cal Z}'_i$ as measuring the progress made by the branching program, towards reaching a state with distribution similar to
$\P_{x|s}$. 

We show that each ${\cal Z}_i$ may only be negligibly larger than ${\cal Z}_{i-1}$. Hence, since it's easy to calculate that ${\cal Z}_0 = 2^{-\frac{2nk'}{2\vareps}}$, it follows that
${\cal Z}_i$  is close to $2^{-\frac{2nk'}{2\vareps}}$, for every $i$.
On the other hand, if $s$ is in layer-$i$ then
${\cal Z}_i$ is at least $\Pr(s) \cdot \langle \P_{x|s},\P_{x|s}\rangle^{\frac{k'}{2\vareps}}$. Thus,
$\Pr(s) \cdot \langle \P_{x|s},\P_{x|s}\rangle^\frac{k'}{2\vareps}$ cannot be much larger than
$2^{-2n\frac{k'}{2\vareps}}$.
Since $s$ is significant,
$\langle \P_{x|s},\P_{x|s}\rangle^\frac{k'}{2\vareps} > 2^{(2\ell' -2n) \frac{k'}{2\vareps}}$
and hence $\Pr(s)$ is at most $2^{-\Omega(\frac{k' \ell'}{\vareps})}$. 

The proof that ${\cal Z}_i$ may only be negligibly larger than ${\cal Z}_{i-1}$
is done in two steps:
Claim~\ref{claim-p2} shows by
a simple convexity argument that ${\cal Z}_i \leq {\cal Z}'_i$. The hard part, that is done in Claim~\ref{claim-p0} and Claim~\ref{claim-p1}, is to prove that
${\cal Z}'_i$ may only be negligibly larger than ${\cal Z}_{i-1}$.

For this proof, we
define for every vertex $v$, the set of edges
$\Gamma_{out}(v)$ that are going out of~$v$, such that $\Pr(e) >0$.
Claim~\ref{claim-p0} shows
that for every vertex $v$,
$$
\sum_{e \in \Gamma_{out}(v)} \Pr(e)
\cdot \langle \P_{x|e},\P_{x|s} \rangle^{k'/2\vareps}$$
may only be negligibly higher than
$$
\Pr(v) \cdot
\langle \P_{x|v},\P_{x|s} \rangle^{k'/2\vareps}.
$$

For the proof of
Claim~\ref{claim-p0}, which is the hardest proof in the paper,
we follow~\cite{Raz17,GRT18} and
consider the function $\P_{x|v} \cdot \P_{x|s}$. We first show how to bound
$\norm{\P_{x|v} \cdot \P_{x|s}}_2$. We then consider two cases:
If $\norm{\P_{x|v} \cdot \P_{x|s}}_1$
is negligible, then
$\langle \P_{x|v},\P_{x|s} \rangle^{k'/2\vareps}$ is negligible and doesn't contribute much,
and we show that for every $e \in \Gamma_{out}(v)$,
$\langle \P_{x|e},\P_{x|s} \rangle^{k'/2\vareps}$ is also negligible and doesn't contribute much.
 If $\norm{\P_{x|v} \cdot\P_{x|s}}_1$ is non-negligible,
we use the bound on $\norm{\P_{x|v} \cdot\P_{x|s}}_2$ and the
assumption that $M$ is a $(k',\ell')$-$L_2$-extractor
to show that for almost all edges $e \in \Gamma_{out}(v)$, we have that
$\langle \P_{x|e},\P_{x|s} \rangle^{k'/2\vareps}$ is very close to
$\langle \P_{x|v},\P_{x|s} \rangle^{k'/2\vareps}$.
Only an exponentially small ($2^{-k'}$) fraction of edges are ``bad'' and give
a significantly larger $\langle \P_{x|e},\P_{x|s} \rangle^{k'/2\vareps}$. In the noiseless case, any ``bad" edge can increase $\langle \P_{x|v},\P_{x|s} \rangle$ by a factor of 2 in the worst case, and hence~\cite{GRT18} raised $\langle \P_{x|v},\P_{x|s} \rangle$ and $\langle \P_{x|e},\P_{x|s} \rangle$ to the power of $k'$, as it is the largest power for which the contribution of the ``bad'' edges is still small (as their fraction is
$2^{-k'}$). But in the noisy case, any ``bad" edge can increase $\langle \P_{x|v},\P_{x|s} \rangle$ by a factor of at most $(1+2\vareps)$ in the worst case, and thus, we can afford to raise $\langle \P_{x|v},\P_{x|s} \rangle$ and $\langle \P_{x|e},\P_{x|s} \rangle$ to the power of $k'/2\vareps$. \emph{This is where our proof differs from that of~\cite{GRT18}}.

This outline oversimplifies many details. To make the argument work, we force $\T$ to stop at significant vertices and whenever $\P_{x|v}(x)$ is large, that is, at significant values, as done in previous papers. And we force $\T$ to stop before traversing some edges, that are so ``bad'' that
their contribution to ${\cal Z}'_i$ is huge and they cannot be ignored.
We show that the total probability that $\T$ stops before reaching a leaf is negligible.

\section{Main Result}\label{sec:main-result}
\begin{theorem} \label{thm:TM-main1}
Let $\tfrac{1}{100}< c <\tfrac{\ln2}{3}$.
Fix $\gamma$
to be such that
$\tfrac{3c}{\ln 2} < \gamma^2 < 1$.
Let $\XX$, $\AA$ be two finite sets.
Let $n = \log_2|\XX|$.
Let $M: \AA \times \XX \rightarrow \{-1,1\}$ be a matrix
which is a $(k',\ell')$-$L_2$-extractor with error $2^{-r'}$,
for sufficiently large\footnote{By {\it ``sufficiently large''} we mean that $k',\ell',r'$ are larger than some  constant that depends on $\gamma$.}
$k',\ell'$ and $r'$, where $\ell' \leq n$.
Let
\begin{equation}
\label{eq:param setting}
\rr :=  \min\left\{ \tfrac{r'}{2}, \tfrac{(1-\gamma)k'}{2}, \tfrac{(1-\gamma)\ell'}{2} -1 \right\}.
\end{equation}
Let $B$ be a branching program, of
length at most $2^{r}$ and width at most $2^{c \cdot k' \cdot \ell'/\vareps}$,
for the learning problem that corresponds to the matrix $M$ with error parameter $\vareps$.
Then,
the success probability of $B$
is at most $O(2^{-r})$.
\end{theorem}


\begin{proof}
We recall the proof in~\cite{GRT18, Raz17} and adapt it to the noisy case. 
Let
\begin{equation}
\label{eq:param setting2}
\kk := \frac{\gamma\ln 2}{2\vareps} k'
\qquad \mbox{and} \qquad
\ell := \gamma \ell'/3.
\end{equation}
Our proof differs from~\cite{GRT18} starting with Claim \ref{claim-d0}, which allows us to set $\kk$ to a larger value of $ \frac{\gamma\ln 2}{2\vareps} k'$ instead of $\gamma(\ln 2) k'$ as set in~\cite{GRT18}.
Note that by the assumption that $k',\ell'$ and $r'$ are sufficiently large, we get that $\kk, \ell$ and~$\rr$  are also sufficiently large.
Since $\ell' \le n$, we have
$\ell + \rr \le \tfrac{\gamma \ell'}{3} +\tfrac{(1-\gamma)\ell'}{2} < \tfrac{\ell'}{2} \le \tfrac{n}{2}$.
Thus,
\begin{equation}\label{eq:rr ell}\rr <n/2-\ell.\end{equation}

Let $B$ be a branching program of
length $m=2^{\rr}$ and width\footnote{width lower bound is vacuous for $\vareps<2^{-r/2}$ as regardless of the width, $\Omega(n/\vareps^2)>2^r$ samples are needed to learn.}
 $d=2^{c \cdot k' \cdot \ell'/\vareps}$ for the learning problem that corresponds to the matrix $M$ with error parameter $\vareps$.
We will show that the success probability of $B$
is at most $O(2^{-\rr})$.


\subsection{The Truncated-Path and Additional Definitions and Notation}

We will define the {\bf truncated-path}, $\T$, to be the same as the computation-path of $B$, except that it sometimes stops before reaching a leaf.
Formally,
we define  $\T$, together with several other definitions and notations, by induction on the layers of the branching program $B$.

Assume that we already defined the truncated-path $\T$, until it reaches layer-$i$ of $B$.
For a vertex $v$ in layer-$i$ of $B$, let $E_v$ be
the event that $\T$ reaches the vertex $v$.
For simplicity, we denote  by $\Pr(v) = \Pr(E_v)$ the probability for $E_v$
(where the probability is over $x,a_1,\ldots,a_m,b_1,\ldots,b_m$), and we denote
by $\P_{x|v} = \P_{x|E_v}$ the distribution of the random variable $x$ conditioned on the event $E_v$.

There will be three cases in which the truncated-path $\T$ stops on a non-leaf $v$:
\begin{enumerate}
\item
If $v$ is a, so called, significant vertex, where the $\ell_2$ norm of $\P_{x|v}$ is non-negligible.
(Intuitively, this means that conditioned on the event that $\T$ reaches~$v$, a non-negligible amount of information is known about $x$).
\item
If $\P_{x|v} (x)$ is non-negligible.
(Intuitively, this means that conditioned on the event that $\T$ reaches~$v$, the correct element $x$ could have been guessed with a non-negligible probability).
\item
If $(M \cdot \P_{x|v}) (a_{i+1})$ is non-negligible.
(Intuitively, this means that
$\T$ is about to traverse a ``{bad}'' edge, which is traversed with a non-negligibly higher or lower probability than probability of traversal under uniform distribution on $x$).
\end{enumerate}

Next, we describe these three cases more formally.

\subsubsection*{Significant Vertices}

We say that a vertex $v$ in layer-$i$ of $B$ is {\bf significant} if
$$
\norm{\P_{x|v}}_{2} > 2^{\ell} \cdot 2^{-n}.
$$

\subsubsection*{Significant Values}

Even if $v$ is not significant, $\P_{x|v}$ may have relatively large values.
For a vertex $v$ in layer-$i$ of~$B$, denote by $\Sig(v)$ the set of all $x' \in \XX$, such that,
$$\P_{x|v}(x') > 2^{2\ell +2\rr} \cdot 2^{-n}.$$

\subsubsection*{Bad Edges}

For a vertex $v$ in layer-$i$ of $B$, denote by $\Bad(v)$ the set of all $\alpha \in \AA$, such that,
$$\left| (M \cdot\P_{x|v})(\alpha) \right|
\geq 2^{-r'}.$$

\subsubsection*{The Truncated-Path $\T$}

We define $\T$ by induction on the layers of the branching program $B$.
Assume that we already defined $\T$ until it reaches a vertex $v$ in layer-$i$ of~$B$.
The path $\T$ stops on $v$ if (at least) one of the following occurs:
\begin{enumerate}
\item
$v$ is significant.
\item
$x \in \Sig(v)$.
\item
$a_{i+1} \in \Bad(v)$.
\item
$v$ is a leaf.
\end{enumerate}
Otherwise, $\T$ proceeds by following the edge labeled by~$(a_{i+1}, b_{i+1})$
(same as the computational-path).

\subsection{Proof of Theorem~\ref{thm:TM-main1}}

Since $\T$ follows the computation-path of $B$, except that it sometimes stops before reaching a leaf, the success probability of $B$ is bounded (from above) by the probability that $\T$ stops before reaching a leaf, plus the probability that $\T$ reaches a leaf $v$ and $\tilde{x}(v) = x$.

The main lemma needed for the proof of Theorem~\ref{thm:TM-main1} is
Lemma~\ref{lemma-main1} that shows that
the probability that $\T$ reaches a significant
vertex is at most $O(2^{-\rr})$.

\begin{lemma} \label{lemma-main1}
The probability that $\T$ reaches a significant vertex is at most $O(2^{-\rr})$.
\end{lemma}

Lemma~\ref{lemma-main1} is proved in Section~\ref{section:mainlemma}. We will now show how the proof of Theorem~\ref{thm:TM-main1} follows from that lemma.

Lemma~\ref{lemma-main1}  shows that the probability that $\T$ stops on a non-leaf vertex, because of the first reason (i.e., that the vertex is significant), is small. The next two lemmas imply that the probabilities that $\T$ stops on a non-leaf vertex, because of the second and third reasons, are also small.

\begin{claim} \label{claim-A0}
If $v$ is a non-significant vertex of $B$ then
$$
\Pr_{x}
[x \in \Sig(v)  \; | \; E_v] \leq 2^{-2\rr}.
$$
\end{claim}

\begin{proof}
Since $v$ is not significant,
$$
\Ex_{x' \sim \P_{x|v}} \left[ \P_{x|v}(x') \right]  =
\sum_{x' \in \XX} \left[ \P_{x|v}(x')^{2} \right] =
2^n \cdot \Ex_{x' \in_R \XX} \left[ \P_{x|v}(x')^{2} \right]
\leq 2^{2\ell} \cdot 2^{-n}.
$$
Hence, by Markov's inequality,
$$
\Pr_{x' \sim \P_{x|v}}
\left[
\P_{x|v}(x') > 2^{2\rr} \cdot  2^{2\ell} \cdot 2^{-n}
\right]
\leq 2^{-2\rr}.
$$
Since conditioned on $E_v$, the distribution of $x$ is $\P_{x|v}$, we obtain
\[
\Pr_{x}
\left[x \in \Sig(v) \; \big| \; E_v\right] =
\Pr_{x}
\left[
\left(\P_{x|v}(x) > 2^{2\rr} \cdot  2^{2\ell} \cdot 2^{-n}\right)
\; \big|\; E_v\;
\right]
\leq 2^{-2\rr}.\qedhere
\]
\end{proof}

\begin{claim} \label{claim-A2}
If $v$ is a non-significant vertex of $B$ then
$$
\Pr_{a_{i+1}} [a_{i+1} \in \Bad(v)] \leq 2^{-2\rr}.
$$
\end{claim}

\begin{proof}
Since $v$ is not significant, $\norm{\P_{x|v}}_2 \le 2^{\ell} \cdot 2^{-n}$.
Since $\P_{x|v}$ is a distribution, $\norm{\P_{x|v}}_1 = 2^{-n}$.
Thus, $$\frac{\norm{\P_{x|v}}_2}{\norm{\P_{x|v}}_1} \le 2^{\ell} \le 2^{\ell'}.$$
Since $M$ is a $(k',\ell')$-$L_2$-extractor with error $2^{-r'}$, there are at most $2^{-k'} \cdot |A|$ elements $\alpha \in \AA$ with
$$
\left|\inner{M_{\alpha}, \P_{x|v}}\right| \ge 2^{-r'} \cdot {\norm{\P_{x|v}}_1} = 2^{-r'} \cdot 2^{-n}$$
The claim follows since $a_{i+1}$ is uniformly distributed over $\AA$ and since $k' \ge 2\rr$ (Equation~\eqref{eq:param setting}).
\end{proof}

We can now use Lemma~\ref{lemma-main1}, Claim~\ref{claim-A0} and Claim~\ref{claim-A2} to prove that the probability that~$\T$ stops before reaching a leaf is at most $O(2^{-\rr})$.
Lemma~\ref{lemma-main1}  shows that the probability that~$\T$ reaches a significant vertex and hence stops because of the first reason, is at most $O(2^{-\rr})$.
Assuming that $\T$ doesn't reach any significant vertex (in which case it would have stopped because of the first reason), Claim~\ref{claim-A0} shows that in each step, the probability that $\T$ stops because of the second reason, is at most $2^{-2\rr}$. Taking a union bound over the $m=2^{\rr}$ steps, the total probability that $\T$ stops because of the second reason, is at most $2^{-\rr}$. In the same way,
assuming that $\T$ doesn't reach any significant vertex (in which case it would have stopped because of the first reason), Claim~\ref{claim-A2} shows that in each step, the probability that $\T$ stops because of the third reason, is at most $2^{-2\rr}$. Again, taking a union bound over the $2^{\rr}$ steps, the total probability that $\T$ stops because of the third reason, is at most $2^{-\rr}$.
Thus, the total probability that~$\T$ stops (for any reason) before reaching a leaf is at most $O(2^{-\rr})$.

Recall that if $\T$ doesn't stop before reaching a leaf, it just follows the computation-path of~$B$.
Recall also that by
Lemma~\ref{lemma-main1},  the probability that $\T$ reaches a significant leaf is at most $O(2^{-\rr})$.
Thus, to bound (from above) the success probability of $B$ by
$O(2^{-\rr})$,
it remains to bound the probability that $\T$ reaches a non-significant leaf $v$
and
$\tilde{x}(v) = x$.
Claim~\ref{claim-A1} shows that for any non-significant leaf $v$, conditioned
on the event that $\T$ reaches~$v$, the probability
for $\tilde{x}(v) = x$ is at most $2^{- \rr}$, which completes the proof
of Theorem~\ref{thm:TM-main1}.

\begin{claim} \label{claim-A1}
If $v$ is a non-significant leaf of $B$ then
$$
\Pr [ \tilde{x}(v) = x \; | \; E_v]
\leq  2^{- \rr}.
$$

\end{claim}
\begin{proof}
Since $v$ is not significant,
$$
\Ex_{x' \in_R \XX} \left[ \P_{x|v}(x')^{2} \right]
\leq 2^{2\ell} \cdot 2^{-2n}.
$$
Hence, for every $x' \in \XX$,
$$
\Pr [x=x' \; | \; E_v]=
\P_{x|v}(x')
\leq 2^{\ell} \cdot 2^{-n/2}
\le 2^{- \rr}
$$
since $\rr \le n/2 - \ell$ (Equation~\eqref{eq:rr ell}).
In particular,
\[
\Pr [\tilde{x}(v) = x \; | \; E_v]
\le 2^{- \rr}.\qedhere
\]
\end{proof}

This completes the proof
of Theorem~\ref{thm:TM-main1}.
\end{proof}

\subsection{Proof of Lemma~\ref{lemma-main1}} \label{section:mainlemma}

\begin{proof}
We need to prove that the probability that $\T$ reaches any significant vertex is at most $O(2^{-\rr})$.
Let $s$ be a  significant vertex of $B$. We will bound from above
the probability that~$\T$ reaches~$s$, and then use
a union bound over all significant vertices of $B$.
Interestingly, the upper bound on the width of $B$ is used only in the union bound.

\subsubsection*{The Distributions $\P_{x|v}$ and  $\P_{x|e}$}

Recall that for a vertex $v$ of $B$, we denote by $E_v$
the event that $\T$ reaches the vertex $v$.
For simplicity, we denote  by $\Pr(v) = \Pr(E_v)$ the probability for $E_v$
(where the probability is over $x,a_1,\ldots,a_m,b_1,...,b_m$), and we denote
by $\P_{x|v} = \P_{x|E_v}$ the distribution of the random variable $x$ conditioned on the event $E_v$.

Similarly,
for an edge~$e$ of the branching program $B$, let $E_e$ be
the event that $\T$ traverses the edge~$e$.
Denote, $\Pr(e) = \Pr(E_e)$
(where the probability is over $x,a_1,\ldots,a_m,b_1,...,b_m$), and
$\P_{x|e} = \P_{x|E_e}$.

\begin{claim} \label{claim-d0}
For any edge~$e = (v,u)$ of $B$, labeled by $(a,b)$, such that
$\Pr(e) > 0$, for any $x' \in \XX$,
$$
\P_{x|e} (x')  = \left\{
\begin{array}{ccccc}
  0
  & \;\;\;\; \mbox{if } & x' \in \Sig(v)  \\
  \P_{x|v} (x') (1+2\vareps) \cdot c_e^{-1}
  & \;\;\;\; \mbox{if } & x' \not \in \Sig(v)& \mbox{and} & M(a,x') = b\\
  \P_{x|v} (x') (1-2\vareps)\cdot c_e^{-1}
  & \;\;\;\; \mbox{if } & x' \not \in \Sig(v)& \mbox{and} & M(a,x') \neq b
\end{array}
\right.
$$
where $c_e$ is a normalization factor that satisfies,
$$
c_e \geq
1-4\cdot 2^{-2\rr}.
$$
\end{claim}
\begin{proof}
Let $e = (v,u)$ be an edge of $B$, labeled by $(a,b)$, and such that
$\Pr(e) > 0$.
Since $\Pr(e) > 0$, the vertex $v$ is not significant (as otherwise $\T$ always stops on $v$ and hence $\Pr(e) = 0$).
Also, since $\Pr(e) > 0$, we know that $a \not \in \Bad(v)$
(as otherwise $\T$ never traverses~$e$ and hence $\Pr(e) = 0$).

If $\T$ reaches $v$, it  traverses the edge $e$ if and only if:
$x \not \in \Sig(v)$ (as otherwise $\T$ stops on~$v$) and $a_{i+1} = a$, $b_{i+1}=b$.
Therefore, by Bayes' rule, for any $x' \in \XX$,
$$
\P_{x|e} (x')  = \left\{
\begin{array}{ccccc}
  0
 & \;\;\;\; \mbox{if } & x' \in \Sig(v)  \\
  \P_{x|v} (x') (1+2\vareps) \cdot c_e^{-1}
  & \;\;\;\; \mbox{if } & x' \not \in \Sig(v)& \mbox{and} & M(a,x') = b\\
  \P_{x|v} (x') (1-2\vareps)\cdot c_e^{-1}
  & \;\;\;\; \mbox{if } & x' \not \in \Sig(v)& \mbox{and} & M(a,x') \neq b
\end{array}
\right.
$$
where $c_e$ is a normalization factor, given by
\begin{align*}
c_e&=
\sum_{\left\{ x' \; : \; x' \not \in \Sig(v) \; \wedge \; M(a,x') = b  \right\} }
\P_{x|v} (x') (1+2\vareps)+\sum_{\left\{ x' \; : \; x' \not \in \Sig(v) \; \wedge \; M(a,x') \neq b  \right\} }
\P_{x|v} (x') (1-2\vareps)
\\& = 
(1+2\vareps)\cdot \Pr_x[(x \not \in \Sig(v)) \wedge  (M(a,x) = b) \; | \;E_v]+(1-2\vareps)\cdot \Pr_x[(x \not \in \Sig(v)) \wedge  (M(a,x) \neq b) \; | \;E_v].
\end{align*}

Since $v$ is not significant, by Claim~\ref{claim-A0},
$$
\Pr_{x}
[x \in \Sig(v)  \; | \; E_v] \leq 2^{-2\rr}.
$$

Since $a \not \in \Bad(v)$,
$$
\left| \Pr_{x} [M(a,x) = 1  \; | \; E_v] -
\Pr_{x} [M(a,x) = -1  \; | \; E_v] \right|
=
\left| (M \cdot\P_{x|v})(a) \right|
\leq 2^{-r'},$$
and hence for every $b'\in\{-1,1\}$,
$$
\Pr_{x}
[M(a,x) = b'  \; | \; E_v] \geq \tfrac{1}{2} - 2^{-r'}.
$$
Hence, by the union bound,
$$
c_e\ge (1+2\vareps)\cdot (\tfrac{1}{2} - 2^{-r'}-2^{-2\rr})+(1-2\vareps)\cdot (\tfrac{1}{2} - 2^{-r'}-2^{-2\rr})\ge 1-4\cdot 2^{-2r}
$$
(where the last inequality follows since
$\rr \le r'/2$, by Equation~\eqref{eq:param setting}).
\end{proof}

\subsubsection*{Bounding the Norm of $\P_{x|s}$}

We will show that $\norm{\P_{x|s}}_{2}$ cannot be too large. Towards this, we will first prove that for every edge $e$ of $B$ that is traversed by $\T$ with probability larger than zero, $\norm{\P_{x|e}}_{2}$ cannot be too large.

\begin{claim} \label{claim-b0}
For any edge~$e$ of $B$, such that
$\Pr(e) > 0$,
$$ \norm{\P_{x|e}}_{2} \leq 4 \cdot 2^{\ell} \cdot 2^{-n}.$$
\end{claim}
\begin{proof}
Let $e = (v,u)$ be an edge of $B$, labeled by $(a,b)$, and such that
$\Pr(e) > 0$.
Since $\Pr(e) > 0$, the vertex $v$ is not significant (as otherwise $\T$ always stops on $v$ and hence $\Pr(e) = 0$).
Thus,
$$
\norm{\P_{x|v}}_{2} \leq 2^{\ell} \cdot 2^{-n}.
$$

By Claim~\ref{claim-d0}, for any $x' \in \XX$,
$$
\P_{x|e} (x')  = \left\{
\begin{array}{ccccc}
  0
  & \;\;\;\; \mbox{if } & x' \in \Sig(v)  \\
  \P_{x|v} (x') (1+2\vareps) \cdot c_e^{-1}
  & \;\;\;\; \mbox{if } & x' \not \in \Sig(v)& \mbox{and} & M(a,x') = b\\
  \P_{x|v} (x') (1-2\vareps)\cdot c_e^{-1}
  & \;\;\;\; \mbox{if } & x' \not \in \Sig(v)& \mbox{and} & M(a,x') \neq b
\end{array}
\right.
$$
where $c_e$ is a normalization factor that satisfies,
$$
c_e \geq
1-4\cdot 2^{-2\rr}>\tfrac{1}{2}.
$$
(where the last inequality holds because we assume that $k',\ell',r'$ and thus $\rr$ are sufficiently large.)
Thus,
\[
\norm{\P_{x|e}}_{2} \leq c_e^{-1} \cdot (1+2\vareps)\norm{\P_{x|v}}_{2} \leq
4 \cdot 2^{\ell} \cdot 2^{-n}\qedhere
\]
\end{proof}

\begin{claim} \label{claim-b1}
$$ \norm{\P_{x|s}}_{2} \leq 4 \cdot 2^{\ell} \cdot 2^{-n}.$$
\end{claim}
\begin{proof}
Let $\Gamma_{in}(s)$ be the set of all edges $e$ of $B$, that are going into $s$, such that $\Pr(e) >0$.
Note that $$\sum_{e \in \Gamma_{in}(s)} \Pr(e) = \Pr(s).$$

By the law of total probability,
for every $x' \in \XX$,
$$
\P_{x|s} (x') =
\sum_{e \in \Gamma_{in}(s)} \tfrac{\Pr(e)}{\Pr(s)} \cdot \P_{x|e} (x'),
$$
and hence by Jensen's inequality,
$$
\P_{x|s} (x')^2 \leq
\sum_{e \in \Gamma_{in}(s)} \tfrac{\Pr(e)}{\Pr(s)} \cdot \P_{x|e} (x')^2.
$$
Summing over $x' \in \XX$, we obtain,
$$
\norm{\P_{x|s}}_{2}^2 \leq
\sum_{e \in \Gamma_{in}(s)} \tfrac{\Pr(e)}{\Pr(s)} \cdot
\norm{\P_{x|e}}_{2}^2.
$$

By Claim~\ref{claim-b0},
for any $e \in \Gamma_{in}(s)$,
$$ \norm{\P_{x|e}}_{2}^2 \leq \left( 4 \cdot 2^{\ell} \cdot 2^{-n} \right)^2.$$
Hence,
\[
\norm{\P_{x|s}}_{2}^2 \leq \left( 4 \cdot 2^{\ell} \cdot 2^{-n} \right)^2.\qedhere
\]
\end{proof}

\subsubsection*{Similarity to a Target Distribution}

Recall that for two functions
$f,g: \XX \rightarrow \Reals^+$, we defined
$$
\langle f,g \rangle = \Ex_{z \in_R \XX} [ f(z) \cdot g(z) ].
$$
We think of $\langle f,g \rangle$ as a measure for the similarity between a function $f$ and a target function~$g$.
Typically $f,g$ will be distributions.

\begin{claim} \label{claim-s1}
$$\langle \P_{x|s},\P_{x|s} \rangle > 2^{2\ell} \cdot 2^{-2n}.$$
\end{claim}
\begin{proof}
Since $s$ is significant,
\[
\langle \P_{x|s},\P_{x|s} \rangle =
\norm{\P_{x|s}}^2_{2} > 2^{2\ell} \cdot 2^{-2n}.\qedhere
\]
\end{proof}

\begin{claim} \label{claim-s2}
$$\langle \U_X,\P_{x|s} \rangle = 2^{-2n},$$
where $\U_X$ is the uniform distribution over $\XX$.
\end{claim}
\begin{proof}
Since $\P_{x|s}$ is a distribution,
\[\langle \U_X,\P_{x|s} \rangle =  2^{-2n} \cdot \sum_{z \in \XX} \P_{x|s}(z) = 2^{-2n}.\qedhere\]
\end{proof}

\subsubsection*{Measuring the Progress}

For $i \in \{0,\ldots ,m\}$, let $L_i$ be the set of vertices $v$ in layer-$i$ of $B$,
such that $\Pr (v) >0$. For $i \in \{1,\ldots ,m\}$,
let $\Gamma_i$ be the set of edges $e$ from layer-$(i-1)$ of $B$ to layer-$i$ of $B$,
such that $\Pr (e) >0$.
Recall that $\kk = \frac{\gamma\ln 2}{2\vareps }k'$ (Equation~\eqref{eq:param setting2}).

For $i \in \{0,\ldots ,m\}$, let
$$
{\cal Z}_i =
\sum_{v \in L_i} \Pr(v) \cdot \langle \P_{x|v},\P_{x|s} \rangle^{\kk}.
$$
For $i \in \{1,\ldots ,m\}$, let
$$
{\cal Z}'_i =
\sum_{e \in \Gamma_i} \Pr(e) \cdot \langle \P_{x|e},\P_{x|s} \rangle^{\kk}.
$$

We think of ${\cal Z}_i, {\cal Z}'_i$ as measuring the progress made by the branching program, towards reaching a state with distribution similar to
$\P_{x|s}$.

For a vertex $v$ of $B$, let
$\Gamma_{out}(v)$ be the set of all edges $e$ of $B$, that are going out of $v$, such that $\Pr(e) >0$.
Note that $$\sum_{e \in \Gamma_{out}(v)} \Pr(e) \leq \Pr(v).$$
(We don't always have an equality here, since sometimes $\T$ stops on $v$).

The next four claims show that the progress made by the branching program is slow.

\begin{claim} \label{claim-p0}
For every vertex $v$ of $B$, such that $\Pr (v) > 0$,
$$
\sum_{e \in \Gamma_{out}(v)} \tfrac{\Pr(e)}{\Pr(v)}
\cdot \langle \P_{x|e},\P_{x|s} \rangle^{\kk}
\leq
\langle \P_{x|v},\P_{x|s} \rangle^{\kk} \cdot
\left( 1 + 2^{-\rr}\right)^{\kk}
+ \left( 2^{-2n +2} \right)^{\kk}.
$$
\end{claim}
\begin{proof}
If $v$ is significant or $v$ is a leaf, then $\T$ always stops on $v$ and hence
$\Gamma_{out}(v)$ is empty and thus the left hand side is equal to zero and the right hand side is positive, so the claim follows trivially.
Thus, we can assume that $v$ is not significant and is not a leaf.

Define $P: \XX \rightarrow \Reals^+$ as follows.
For any $x' \in \XX$,
$$
P (x')  = \left\{
\begin{array}{ccc}
  0
  & \;\;\;\; \mbox{if } & x' \in \Sig(v)  \\
  \P_{x|v} (x')
  & \;\;\;\; \mbox{if } & x' \not \in \Sig(v)
\end{array}
\right.
$$
Note that by the definition of
$\Sig(v)$,  for any $x' \in \XX$,
\begin{equation} \label{e11}
P(x') \leq 2^{2\ell + 2\rr} \cdot 2^{-n}.
\end{equation}

Define $f: \XX \rightarrow \Reals^+$ as follows.
For any $x' \in \XX$,
$$
f(x')  = P(x') \cdot  \P_{x|s}(x').
$$
By Claim~\ref{claim-b1} and Equation~\eqref{e11},
\begin{equation} \label{e12}
\norm{f}_{2} \leq
2^{2\ell + 2\rr} \cdot 2^{-n} \cdot
\norm{\P_{x|s}}_{2} \leq
2^{2\ell + 2\rr}
\cdot 2^{-n} \cdot 4 \cdot 2^{\ell} \cdot 2^{-n}
= 2^{3\ell + 2\rr +2}
\cdot 2^{-2n}.
\end{equation}

By Claim~\ref{claim-d0},
for any edge~$e \in \Gamma_{out}(v)$, labeled by $(a,b)$, for any $x' \in \XX$,
$$
\P_{x|e} (x')  = \left\{
\begin{array}{ccccc}
  0
  & \;\;\;\; \mbox{if } & x' \in \Sig(v)  \\
  \P_{x|v} (x') (1+2\vareps) \cdot c_e^{-1}
  & \;\;\;\; \mbox{if } & x' \not \in \Sig(v)& \mbox{and} & M(a,x') = b\\
  \P_{x|v} (x') (1-2\vareps)\cdot c_e^{-1}
  & \;\;\;\; \mbox{if } & x' \not \in \Sig(v)& \mbox{and} & M(a,x') \neq b
\end{array}
\right.
$$
where $c_e$ is a normalization factor that satisfies,
$$
c_e \geq
1-4\cdot 2^{-2\rr}.
$$
Therefore,
for any edge~$e \in \Gamma_{out}(v)$, labeled by $(a,b)$, for any $x' \in \XX$,
$$
\P_{x|e} (x') \cdot \P_{x|s} (x') = 
  f(x') \cdot (1+2\vareps \cdot b\cdot M(a,x'))\cdot c_e^{-1}
  $$
and hence, we have
\begin{align}
\nonumber
\langle \P_{x|e},\P_{x|s} \rangle &=
\Ex_{x' \in_R \XX} [ \P_{x|e}(x') \cdot \P_{x|s}(x') ]
=
\Ex_{x' \in_R \XX}
[ f(x') \cdot (1+2\vareps \cdot b\cdot M(a,x'))\cdot c_e^{-1}]
\\
\nonumber
&=
\left(\norm{f}_1 + 2\vareps \cdot b \cdot \inner{M_a, f}\right) \cdot (c_e)^{-1}
\\
\label{e13}
&<
\left( \norm{f}_1+2\vareps |\inner{M_a, f}| \right)
\cdot \left(1 + 2^{-2\rr + 3}\right)
\end{align}
(where the last inequality holds by the bound that we have on $c_e$, because we assume that $k',\ell',r'$ and thus $\rr$ are sufficiently large).

We will now consider two cases:

\subsubsection*{Case I: $\norm{f}_1 <  2^{-2n}$}
In this case, we bound  $|\inner{M_a, f}| \leq \norm{f}_1$
(since $f$ is non-negative and the entries of $M$ are in~$\{-1,1\}$)
and
$(1 + 2^{-2\rr +3}) < 2$ (since we assume that $k',\ell',r'$ and thus $\rr$ are sufficiently large) and obtain
for any edge~$e \in \Gamma_{out}(v)$,
$$\langle \P_{x|e},\P_{x|s} \rangle
< 4  \cdot 2^{-2n}.$$
Since $\sum_{e \in \Gamma_{out}(v)} \tfrac{\Pr(e)}{\Pr(v)} \leq 1$,
Claim~\ref{claim-p0} follows, as the left hand side of the claim is smaller than the second term on the right hand side.

\subsubsection*{Case II: $\norm{f}_1 \geq 2^{-2n}$}
For every $a \in \AA$, define
$$t(a) = \frac{|\inner{M_a, f}|}{\norm{f}_1}.$$
By Equation~\eqref{e13},
\begin{equation}\label{e14}
\langle \P_{x|e},\P_{x|s} \rangle^{\kk} <
\norm{f}_1^{\kk} \cdot
\left( 1 + 2\vareps\cdot t(a) \right)^{\kk}
\cdot \left( 1 + 2^{-2\rr + 3}\right)^{\kk}.
\end{equation}

Note that by the definitions of $P$ and $f$,
$$
\norm{f}_1 = \Ex_{x' \in_R \XX} [f(x')] =
\langle P,\P_{x|s} \rangle \leq \langle \P_{x|v},\P_{x|s} \rangle.
$$
Note also that for every $a \in \AA$, there is at most one edge
$e_{(a,1)} \in \Gamma_{out}(v)$, labeled by $(a,1)$, and
at most one edge
$e_{(a,-1)} \in \Gamma_{out}(v)$, labeled by $(a,-1)$,
and we have
$$\tfrac{\Pr(e_{(a,1)})}{\Pr(v)} + \tfrac{\Pr(e_{(a,-1)})}{\Pr(v)}
\leq \tfrac{1}{|A|},$$
since $\tfrac{1}{|A|}$ is the probability that the next sample read by the program is $a$.
%
Thus, summing over all $e \in \Gamma_{out}(v)$, by Equation~\eqref{e14},
\begin{equation}\label{e15}
\sum_{e \in \Gamma_{out}(v)}
\tfrac{\Pr(e)}{\Pr(v)}  \cdot \langle \P_{x|e},\P_{x|s} \rangle^{\kk} <
\langle \P_{x|v},\P_{x|s} \rangle ^{\kk} \cdot
\Ex_{a \in_R \AA} \left[ \left( 1 + 2\vareps\cdot t(a) \right)^{\kk} \right]
\cdot \left( 1 + 2^{-2\rr +3}\right)^{\kk}.
\end{equation}

It remains to bound
\begin{equation} \label{e16}
\Ex_{a \in_R \AA} \left[ \left( 1 +2\vareps\cdot t(a)\right)^{\kk} \right],
\end{equation}
using the properties of the matrix $M$ and the bounds on the $\ell_2$ versus $\ell_1$ norms of $f$.

By Equation~\eqref{e12}, the assumption that $\norm{f}_1 \ge 2^{-2n}$, Equation~\eqref{eq:param setting} and Equation~\eqref{eq:param setting2}, we get
\begin{equation*} 
\frac{\norm{f}_{2}}{\norm{f}_1} \leq
2^{3\ell + 2\rr +2} \le 2^{\ell'}\;.
\end{equation*}
Since $M$ is a $(k',\ell')$-$L_2$-extractor with error $2^{-r'}$,
 there are at most $2^{-k'} \cdot |A|$ rows $a\in \AA$ with
$t(a)  = \frac{|\inner{M_a,f}|}{\norm{f}_1} \ge 2^{-r'}$.
We bound the expectation in Equation~\eqref{e16}, by splitting
the expectation into two sums
\begin{equation}
	\label{e18}
\Ex_{a \in_R \AA} \left[ \left( 1 + 2\vareps\cdot t(a) \right)^{\kk} \right]
= \tfrac{1}{|A|} \cdot  \sum_{a \;: \; t(a) \leq 2^{-r'}}
\left( 1 + 2\vareps\cdot t(a) \right)^{\kk}
+
\tfrac{1}{|A|} \cdot  \sum_{a \;: \; t(a) > 2^{-r'}}
\left( 1 + 2\vareps\cdot t(a)\right)^{\kk}.
\end{equation}

We bound the first sum in Equation~\eqref{e18} by $(1+2\vareps\cdot 2^{-r'})^{\kk}$.
As for the second sum in Equation~\eqref{e18}, we
know that it is a sum of at most $2^{-k'} \cdot |\AA|$ elements,
and since for every $a \in \AA$, we have $t(a) \leq 1$, we have
\begin{equation*} 
\tfrac{1}{|A|} \cdot  \sum_{a \;: \; t(a) > 2^{-r'}}
\left( 1 + 2\vareps\cdot t(a)\right)^{\kk} \leq
2^{-k'} \cdot (1+2\vareps)^{\kk}
\le 2^{-k'}e^{2\vareps k}
\le
2^{-2\rr}\;
\end{equation*}
(where in the last inequality we used Equations~\eqref{eq:param setting} and~\eqref{eq:param setting2}).
Overall, using Equation~\eqref{eq:param setting} again, we get
\begin{equation}\label{e18.5}
\Ex_{a \in_R \AA} \left[ \left( 1 +2\vareps\cdot  t(a) \right)^{\kk} \right] \le (1+2\vareps\cdot 2^{-r'})^{\kk} + 2^{-2\rr} \le (1+2^{-2\rr})^{\kk+1}.
\end{equation}
Substituting
Equation~\eqref{e18.5} into Equation~\eqref{e15}, we obtain
\begin{align*}
\sum_{e \in \Gamma_{out}(v)}
\tfrac{\Pr(e)}{\Pr(v)}  \cdot \langle \P_{x|e},\P_{x|s} \rangle^{\kk} &<
\langle \P_{x|v},\P_{x|s} \rangle ^{\kk}
\cdot  \left(1 + 2^{-2\rr}\right)^{\kk+1}
\cdot \left( 1 + 2^{-2\rr +3}\right)^{\kk}
\\
&<
\langle \P_{x|v},\P_{x|s} \rangle ^{\kk}
\cdot \left( 1 + 2^{-\rr}\right)^{\kk}
\end{align*}
(where the last inequality uses the assumption that $\rr$ is sufficiently large).
This completes the proof of Claim~\ref{claim-p0}.
\end{proof}

\begin{claim} \label{claim-p1}
For every $i \in \{1,\ldots ,m\}$,
$${\cal Z}'_i \leq  {\cal Z}_{i-1}
\cdot
\left( 1 + 2^{-\rr}\right)^{\kk}
+ \left( 2^{-2n +2} \right)^{\kk}.
$$\end{claim}
\begin{proof}
By Claim~\ref{claim-p0},
\begin{align*}
{\cal Z}'_i =
\sum_{e \in \Gamma_i} \Pr(e) \cdot \langle \P_{x|e},\P_{x|s} \rangle^{\kk}
&=
\sum_{v \in L_{i-1}} \Pr(v) \cdot
\sum_{e \in \Gamma_{out}(v)} \tfrac{\Pr(e)}{\Pr(v)}
\cdot \langle \P_{x|e},\P_{x|s} \rangle^{\kk}
\\&\leq
\sum_{v \in L_{i-1}} \Pr(v) \cdot
\left(
\langle \P_{x|v},\P_{x|s} \rangle^{\kk} \cdot
\left( 1 + 2^{-\rr}\right)^{\kk}
+ \left( 2^{-2n +2} \right)^{\kk}
\right)
\\&=
{\cal Z}_{i-1} \cdot
\left( 1 + 2^{-\rr}\right)^{\kk} +
\sum_{v \in L_{i-1}} \Pr(v) \cdot
\left( 2^{-2n +2} \right)^{\kk}
\\&\leq
{\cal Z}_{i-1} \cdot
\left( 1 + 2^{-\rr}\right)^{\kk} +
\left( 2^{-2n +2} \right)^{\kk}\qedhere
\end{align*}
\end{proof}

\begin{claim} \label{claim-p2}
For every $i \in \{1,\ldots ,m\}$,
$$
{\cal Z}_{i} \leq
{\cal Z}'_i.
$$
\end{claim}
\begin{proof}
For any $v \in L_i$,
let $\Gamma_{in}(v)$ be the set of all edges $e \in \Gamma_i$, that are going into $v$.
Note that $$\sum_{e \in \Gamma_{in}(v)} \Pr(e) = \Pr(v).$$

By the law of total probability,
for every $v \in L_i$ and every $x' \in \XX$,
$$
\P_{x|v} (x') =
\sum_{e \in \Gamma_{in}(v)} \tfrac{\Pr(e)}{\Pr(v)} \cdot \P_{x|e} (x'),
$$
and hence
$$
\langle \P_{x|v},\P_{x|s} \rangle =
\sum_{e \in \Gamma_{in}(v)} \tfrac{\Pr(e)}{\Pr(v)} \cdot
\langle \P_{x|e},\P_{x|s} \rangle.
$$
Thus,
by Jensen's inequality,
$$
\langle \P_{x|v},\P_{x|s} \rangle  ^{\kk}
\leq
\sum_{e \in \Gamma_{in}(v)} \tfrac{\Pr(e)}{\Pr(v)} \cdot
\langle \P_{x|e},\P_{x|s} \rangle ^{\kk}.
$$

Summing over all $v \in L_i$, we get
$$
{\cal Z}_{i} =
\sum_{v \in L_{i}} \Pr(v) \cdot
\langle \P_{x|v},\P_{x|s} \rangle  ^{\kk}
\leq
\sum_{v \in L_{i}} \Pr(v) \cdot
\sum_{e \in \Gamma_{in}(v)} \tfrac{\Pr(e)}{\Pr(v)} \cdot
\langle \P_{x|e},\P_{x|s} \rangle ^{\kk}
$$
\[
=
\sum_{e \in \Gamma_i} \Pr(e) \cdot \langle \P_{x|e},\P_{x|s} \rangle^{\kk}
=
{\cal Z}'_i.\qedhere
\]
\end{proof}

\begin{claim} \label{claim-p3}
For every $i \in \{1,\ldots ,m\}$,
$$
{\cal Z}_i \leq
2^{4\kk + 2\rr} \cdot 2^{-2\kk \cdot n}.
$$
\end{claim}

\begin{proof}
By Claim~\ref{claim-s2}, ${\cal Z}_0 = (2^{-2n})^{\kk}$.
By Claim~\ref{claim-p1} and Claim~\ref{claim-p2}, for every
$i \in \{1,\ldots ,m\}$,
$$
{\cal Z}_i \leq  {\cal Z}_{i-1}
\cdot
\left( 1 + 2^{-\rr}\right)^{\kk}
+ \left( 2^{-2n +2} \right)^{\kk}.
$$
Hence, for every
$i \in \{1,\ldots ,m\}$,
$$
{\cal Z}_i \leq  \left( 2^{-2n +2} \right)^{\kk}
\cdot
(m+1) \cdot
\left( 1 + 2^{-\rr}\right)^{\kk m}.
$$
Since $m = 2^{\rr}$,
\[
{\cal Z}_i \leq
2^{-2\kk \cdot n} \cdot 2^{2\kk} \cdot (2^{\rr} +1)\cdot e^{\kk}
\leq
2^{-2\kk \cdot n} \cdot 2^{4\kk + 2\rr}.\qedhere\]
\end{proof}

\subsubsection*{Proof of Lemma~\ref{lemma-main1}}

We can now complete the proof of Lemma~\ref{lemma-main1}.
Assume that $s$ is in layer-$i$ of $B$.
By Claim~\ref{claim-s1},
$${\cal Z}_i \geq \Pr(s) \cdot \langle \P_{x|s},\P_{x|s} \rangle ^{\kk}
> \Pr(s) \cdot \left( 2^{2\ell} \cdot 2^{-2n} \right)^{\kk}
= \Pr(s) \cdot 2^{2\ell \cdot \kk} \cdot 2^{-2\kk \cdot n}.$$
On the other hand, by Claim~\ref{claim-p3},
$$
{\cal Z}_i \leq  2^{4\kk + 2\rr} \cdot 2^{-2\kk \cdot n}.
$$
Thus, using Equation~\eqref{eq:param setting} and Equation~\eqref{eq:param setting2}, we get
$$
\Pr(s) \leq
2^{4\kk + 2\rr} \cdot
2^{-2\ell \cdot \kk}
\le 2^{\frac{2 k'}{\vareps}} \cdot 2^{-\frac{ \gamma^2\ln 2}{3\vareps} (k'\ell')}.
$$


Recall that we assumed that the width of $B$ is at most $2^{c k' \ell'/\vareps}$ for some constant $c<\ln 2/3$,
and that the length of $B$ is at most $2^{\rr}$.
Recall that we fixed
$\gamma$ such that $\gamma^2(\ln2)/3 > c$.
Taking a union bound over at most $2^{\rr} \cdot 2^{c k' \ell'/\vareps} \le 2^{k'} \cdot 2^{c k' \ell'/\vareps}$ significant vertices of $B$, we conclude that the probability that $\T$ reaches any significant vertex is at most $2^{-\Omega(k' \ell'/\vareps)}$.
Since we assume that $k'$ and $\ell'$ are sufficiently large,
$2^{-\Omega(k' \ell'/\vareps)}$
is certainly at most $2^{-k'}$, which is at most $2^{-\rr}$.
\end{proof}

\begin{corollary} \label{cor:main1}
Let $\XX$, $\AA$ be two finite sets.
Let $M: \AA \times \XX \rightarrow \{-1,1\}$ be a matrix.
Assume that $k,\ell, r \in \N$ are large enough and such that any submatrix of $M$ of at least $2^{-k} \cdot |A|$ rows and at least
$2^{-\ell} \cdot |X|$ columns, has a bias of at most $2^{-r}$.

Then,
any learning algorithm for the learning problem corresponding to $M$ with error parameter $\vareps$, requires either a memory
of size at least
$\Omega\left(\frac{k \cdot  \ell}{\vareps} \right)$, or at least $2^{\Omega(r)}$ samples.
The result holds even if the learner has an exponentially small success probability (of $2^{-\Omega(r)}$).
\end{corollary}
Corollary follows from the equivalence between $L_2$-Extractors and $L_\infty$-Extractors (up to constant factors) observed in~\cite{GRT18}.

\section{Acknowledgements}
We would like to thank Avishay Tal and Greg Valiant for the helpful discussions.
\bibliographystyle{alpha}
\bibliography{biblo}

\begin{thebibliography}{KKMS08}

\bibitem[Ale03]{DBLP:conf/focs/Alekhnovich03}
Michael Alekhnovich.
\newblock More on average case vs approximation complexity.
\newblock In {\em 44th Symposium on Foundations of Computer Science {(FOCS}
  2003), 11-14 October 2003, Cambridge, MA, USA, Proceedings}, pages 298--307.
  {IEEE} Computer Society, 2003.

\bibitem[BGY18]{BGY18}
Paul Beame, Shayan~Oveis Gharan, and Xin Yang.
\newblock Time-space tradeoffs for learning finite functions from random
  evaluations, with applications to polynomials.
\newblock In {\em Conference On Learning Theory}, pages 843--856, 2018.

\bibitem[BKW03]{DBLP:journals/jacm/BlumKW03}
Avrim Blum, Adam Kalai, and Hal Wasserman.
\newblock Noise-tolerant learning, the parity problem, and the statistical
  query model.
\newblock {\em J. {ACM}}, 50(4):506--519, 2003.

\bibitem[DKS19]{dks19}
Yuval Dagan, Gil Kur, and Ohad Shamir.
\newblock Space lower bounds for linear prediction in the streaming model.
\newblock In {\em Conference on Learning Theory}, pages 929--954. PMLR, 2019.

\bibitem[DS18]{DS}
Yuval Dagan and Ohad Shamir.
\newblock Detecting correlations with little memory and communication.
\newblock In {\em Conference On Learning Theory}, pages 1145--1198, 2018.

\bibitem[DTZ20]{Stefano20}
Wei Dai, Stefano Tessaro, and Xihu Zhang.
\newblock Super-linear time-memory trade-offs for symmetric encryption.
\newblock Cryptology ePrint Archive, Report 2020/663, 2020.
\newblock \url{https://eprint.iacr.org/2020/663}.

\bibitem[FGKP09]{DBLP:journals/siamcomp/FeldmanGKP09}
Vitaly Feldman, Parikshit Gopalan, Subhash Khot, and Ashok~Kumar Ponnuswami.
\newblock On agnostic learning of parities, monomials, and halfspaces.
\newblock {\em {SIAM} J. Comput.}, 39(2):606--645, 2009.

\bibitem[GRT18]{GRT18}
Sumegha Garg, Ran Raz, and Avishay Tal.
\newblock Extractor-based time-space lower bounds for learning.
\newblock In {\em Proceedings of the 50th Annual ACM SIGACT Symposium on Theory
  of Computing}, pages 990--1002. ACM, 2018.

\bibitem[GRT19]{GRT19}
Sumegha Garg, Ran Raz, and Avishay Tal.
\newblock Time-space lower bounds for two-pass learning.
\newblock In {\em 34th Computational Complexity Conference (CCC 2019)}. Schloss
  Dagstuhl-Leibniz-Zentrum fuer Informatik, 2019.

\bibitem[GRZ20]{girish20}
Uma Girish, Ran Raz, and Wei Zhan.
\newblock Quantum logspace algorithm for powering matrices with bounded norm.
\newblock {\em arXiv preprint arXiv:2006.04880}, 2020.

\bibitem[GZ19]{Zhandry19}
Jiaxin Guan and Mark Zhandary.
\newblock Simple schemes in the bounded storage model.
\newblock In {\em Annual International Conference on the Theory and
  Applications of Cryptographic Techniques}, pages 500--524. Springer, 2019.

\bibitem[GZ21]{GZ21}
Jiaxin Guan and Mark Zhandry.
\newblock Disappearing cryptography in the bounded storage model.
\newblock {\em IACR Cryptol. ePrint Arch.}, 2021:406, 2021.

\bibitem[JT19]{Stefano19}
Joseph Jaeger and Stefano Tessaro.
\newblock Tight time-memory trade-offs for symmetric encryption.
\newblock In {\em Annual International Conference on the Theory and
  Applications of Cryptographic Techniques}, pages 467--497. Springer, 2019.

\bibitem[Kea98]{sq}
Michael~J. Kearns.
\newblock Efficient noise-tolerant learning from statistical queries.
\newblock {\em J. {ACM}}, 45(6):983--1006, 1998.

\bibitem[KKMS08]{DBLP:journals/siamcomp/KalaiKMS08}
Adam~Tauman Kalai, Adam~R. Klivans, Yishay Mansour, and Rocco~A. Servedio.
\newblock Agnostically learning halfspaces.
\newblock {\em {SIAM} J. Comput.}, 37(6):1777--1805, 2008.

\bibitem[KRT17]{KRT}
Gillat Kol, Ran Raz, and Avishay Tal.
\newblock Time-space hardness of learning sparse parities.
\newblock In {\em Proceedings of the 49th Annual ACM SIGACT Symposium on Theory
  of Computing}, pages 1067--1080. ACM, 2017.

\bibitem[MM17]{MM17}
Dana Moshkovitz and Michal Moshkovitz.
\newblock Mixing implies lower bounds for space bounded learning.
\newblock In {\em Conference on Learning Theory}, pages 1516--1566. PMLR, 2017.

\bibitem[MM18]{MM2}
Dana Moshkovitz and Michal Moshkovitz.
\newblock Entropy samplers and strong generic lower bounds for space bounded
  learning.
\newblock In {\em 9th Innovations in Theoretical Computer Science Conference
  (ITCS 2018)}. Schloss Dagstuhl-Leibniz-Zentrum fuer Informatik, 2018.

\bibitem[MT17]{MT}
Michal Moshkovitz and Naftali Tishby.
\newblock Mixing complexity and its applications to neural networks.
\newblock {\em arXiv preprint arXiv:1703.00729}, 2017.

\bibitem[Raz16]{Raz16}
Ran Raz.
\newblock Fast learning requires good memory: A time-space lower bound for
  parity learning.
\newblock In {\em Foundations of Computer Science (FOCS), 2016 IEEE 57th Annual
  Symposium on}, pages 266--275. IEEE, 2016.

\bibitem[Raz17]{Raz17}
Ran Raz.
\newblock A time-space lower bound for a large class of learning problems.
\newblock In {\em 58th {IEEE} Annual Symposium on Foundations of Computer
  Science, {FOCS} 2017, Berkeley, CA, USA, October 15-17, 2017}, pages
  732--742, 2017.

\bibitem[Sha14]{Shamir}
Ohad Shamir.
\newblock Fundamental limits of online and distributed algorithms for
  statistical learning and estimation.
\newblock {\em Advances in Neural Information Processing Systems}, 27:163--171,
  2014.

\bibitem[SSV19]{SSV19}
Vatsal Sharan, Aaron Sidford, and Gregory Valiant.
\newblock Memory-sample tradeoffs for linear regression with small error.
\newblock In {\em Proceedings of the 51st Annual ACM SIGACT Symposium on Theory
  of Computing}, pages 890--901, 2019.

\bibitem[SVW16]{SVW}
Jacob Steinhardt, Gregory Valiant, and Stefan Wager.
\newblock Memory, communication, and statistical queries.
\newblock In {\em Conference on Learning Theory}, pages 1490--1516. PMLR, 2016.

\bibitem[TT18]{Stefano18}
Stefano Tessaro and Aishwarya Thiruvengadam.
\newblock Provable time-memory trade-offs: symmetric cryptography against
  memory-bounded adversaries.
\newblock In {\em Theory of Cryptography Conference}, pages 3--32. Springer,
  2018.

\bibitem[VV16]{VV}
Gregory Valiant and Paul Valiant.
\newblock Information theoretically secure databases.
\newblock {\em arXiv preprint arXiv:1605.02646}, 2016.

\end{thebibliography}
\end{document}